\documentclass[letterpaper]{article} 
\usepackage{aaai2026}  
\usepackage{times}  
\usepackage{helvet}  
\usepackage{courier}  
\usepackage[hyphens]{url}  
\usepackage{graphicx} 
\usepackage{natbib}  
\usepackage{caption} 
\usepackage{algorithm}
\usepackage{algorithmic}
\usepackage{glossaries}
\usepackage{newfloat}
\usepackage{listings}

\usepackage{amssymb}
\usepackage{amsmath}
\usepackage{amsthm}

\usepackage{booktabs}
\usepackage{enumitem}

\usepackage{subcaption}

\usepackage{amsthm}
\newtheorem{example}{Example}

\newtheorem{theorem}{Theorem}

\newtheorem{proposition}[theorem]{Proposition}

\newtheorem{definition}{Definition}

\DeclareCaptionStyle{ruled}{labelfont=normalfont,labelsep=colon,strut=off} 
\floatstyle{ruled}
\newfloat{listing}{tb}{lst}{}
\floatname{listing}{Listing}
\title{Inferring Reward Machines and Transition Machines from \\Partially Observable Markov Decision Processes}
\author{
    Yuly Wu\textsuperscript{\rm 1},
    Jiamou Liu\textsuperscript{\rm 1}\thanks{Corresponding author.},
    Libo Zhang\textsuperscript{\rm 1}
}
\affiliations{
    \textsuperscript{\rm 1}The University of Auckland\\
    \{jwu623, lzha797\}@aucklanduni.ac.nz, jiamou.liu@auckland.ac.nz
}

\usepackage{bibentry}

\nocopyright

\begin{document}

\newacronym{RL}{RL}{Reinforcement Learning}
\newacronym{MDP}{MDP}{Markov Decision Process}
\newacronym{MDP-RM}{MDP-RM}{Markov Decision Process with Reward Machine} 
\newacronym{POMDP}{POMDP}{Partially Observable Markov Decision Process} 
\newacronym{Det-PONMRDP}{Det-PONMRDP}{Deterministic Partially Observable Non-Markovian Reward Decision Process} 
\newacronym{TA}{TA}{Task Automaton} 
\newacronym{HMM}{HMM}{Hidden Markov Model} 
\newacronym{DFA}{DFA}{Deterministic Finite Automaton} 
\newacronym{ILP}{ILP}{Integer Linear Programming} 
\newacronym{NMRDP}{NMRDP}{Non-Markovian Reward Decision Process} 
\newacronym{ARMDP}{ARMDP}{Abstract Reward Markov Decision Process} 
\newacronym{RM}{RM}{Reward Machine} 
\newacronym{TM}{TM}{Transition Machine} 
\newacronym{AP}{AP}{the set of Atomic Propositions} 
\newacronym{NFA}{NFA}{Nondeterministic Finite Automaton} 
\newacronym{DetPOMDP}{Det-POMDP}{Deterministic Partially Observable Markov Decision Process}
\newacronym{DBMM}{DBMM}{Dual-Behavior Mealy Machin}

\maketitle

\begin{abstract}
Partially Observable Markov Decision Processes (POMDPs) are fundamental to many real-world applications. Although reinforcement learning (RL) has shown success in fully observable domains, learning policies from traces in partially observable environments remains challenging due to non-Markovian observations. Inferring an automaton to handle the non-Markovianity is a proven effective approach, but faces two limitations: 1) existing automaton representations focus only on reward-based non-Markovianity, leading to unnatural problem formulations; 2) inference algorithms face enormous computational costs.
For the first limitation, we introduce Transition Machines (TMs) to complement existing Reward Machines (RMs). To develop a unified inference algorithm for both automata types, we propose the Dual Behavior Mealy Machine (DBMM) that subsumes both TMs and RMs. We then introduce DB-RPNI, a passive automata learning algorithm that efficiently infers DBMMs while avoiding the costly reductions required by prior work. We further develop optimization techniques and identify sufficient conditions for inferring the minimal correct automata. Experimentally, our inference method achieves speedups of up to three orders of magnitude over SOTA baselines.
\end{abstract}

\section{Introduction}
Partially Observable Markov Decision Processes (POMDPs) form a foundational framework in reinforcement learning (RL) for decision-making under uncertainty, where agents must make decisions without complete state information. The setting arises in numerous real-world domains, from autonomous navigation to medical diagnosis with incomplete patient data \cite{kaelbling1998planning,cassandra1998survey}. The core challenge lies in the breakdown of the Markov property: identical observations may require different optimal actions depending on the history of interactions. To address this, {\em reward machines (RM)} \cite{icarte2018using} have emerged as a promising approach, serving as external memory structures that capture temporal dependencies in reward functions and restore Markovian properties to enable the application of standard RL algorithms \cite{toro2019learning}. However, the effectiveness of this approach hinges on constructing RMs that accurately capture the underlying temporal structure of the environment.

Inferring an appropriate RM from observed experience traces presents a central challenge. Although one might intuitively seek a minimal automaton, as it most concisely expresses the non-Markovianity of the environment, \cite{toro2019learning} demonstrated that minimal RMs are not necessarily effective representations in POMDPs. Consequently, inferring a RM is formulated as an optimization problem. Moreover, current approaches for inferring such automata in partially observable domains \cite{abate2023learning, hyde2024detecting,toro2019learning} suffer from high computational complexity, limiting their practicality.

Our key insight is that POMDPs exhibit two distinct types of non-Markovian behavior: {\em transition dependencies}, where the next unobserved state depends on hidden historical events; and {\em reward dependencies}, where reward signals depend on past context. To avoid formulating the problem as a complex optimization task, we introduce {\em Transition Machines (TM)} to explicitly model transition dependencies analogously to how RMs capture reward dependencies. The framework naturally decouples the learning problem into inferring minimal TMs and minimal RMs.

To develop a unified inference algorithm for both automata, we introduce the Dual Behavior Mealy Machine (DBMM), a unified representation framework that subsumes both TMs and RMs under a single formalism. Then, to bypass the costly optimization of prior work, we design a direct inference algorithm to infer DBMMs\cite{oncina1992inferring}. We further develop specialized preprocessing and an efficient method, improving efficiency while maintaining theoretical correctness guarantees.

Our experimental evaluation on a range of challenging POMDP environments demonstrates the superiority of our approach. Compared to SOTA baselines~\cite{abate2023learning,hyde2024detecting}, our method achieves speedups of up to three orders of magnitude. Ablation studies validate the contribution of each component: our observation supplement techniques significantly reduce the complexity of the learned automata, while our reduction rules provide substantial computational savings. Finally, the agent augmented with our inferred automata converges to the optimal policy in a 25x25 grid-world partially observable environment.

\section{Related Work}
To address non-Markovian challenges in RL, automaton-based methods have become a prominent approach, theoretically equivalent to linear temporal logic~\cite{vardi1986automata}. Within the automaton-based framework, different works have employed various automaton representations as memory structures. Early approaches utilized Deterministic Finite Automata (DFAs) to capture task dependencies~\cite{neider2021advice,abate2023learning}. Some works~\cite{rens2020online,xu2021active,rens2020learning} have proposed simplified variants such as mealy reward machines that restrict reward machines to Mealy machines. The most prominent approach has been using Reward Machines, which provide explicit encoding of non-Markovian reward structures~\cite{hyde2024detecting,toro2019learning}.

A critical challenge lies in how these automata are constructed. This problem is typically divided into two settings: active learning, where an agent can query an oracle or interact with the environment~\cite{abate2023learning}, and passive learning, where the automaton must be inferred solely from a fixed set of pre-collected traces~\cite{abate2023learning,hyde2024detecting}. While the passive learning setting is often more challenging and thus less frequently studied than active learning, it is of greater practical significance due to its applicability in scenarios with limited interaction.

Several passive learning approaches have been proposed to infer automata from traces. Toro et al.~\cite{toro2019learning} introduced a local search-based method for inferring reward machines in POMDPs, identifying that minimal reward machines are not necessarily optimal in POMDPs and formulating the inference problem as an optimization task. Abate and Brafman~\cite{abate2023learning} proposed an HMM-based approach that learns a DFA as reward representation by reducing the automaton inference problem to hidden Markov model parameter estimation. Hyde et al.~\cite{hyde2024detecting} formulates reward machine learning as an Integer Linear Programming (ILP) problem, encoding structural constraints and data consistency requirements as constraints. These methods all operate by reducing the automaton inference problem to a more computationally demanding, mathematical formalism. Gaon and Brafman~\cite{gaon2020reinforcement} pioneered the integration of classical passive automaton learning techniques from the grammatical inference field into RL. However, the work was limited to inferring DFAs, and was thus incapable of modeling more expressive structures. Other works such as Cipollone et al.~\cite{cipollone2023provably} and Deb et al.~\cite{deb2024tractable} also infer automata from traces, but differ in their goal, primarily using PDFAs to reconstruct a full environmental model instead of for state augmentation.
 
\begin{figure}[htbp]
    \centering
    \includegraphics[width=\columnwidth]{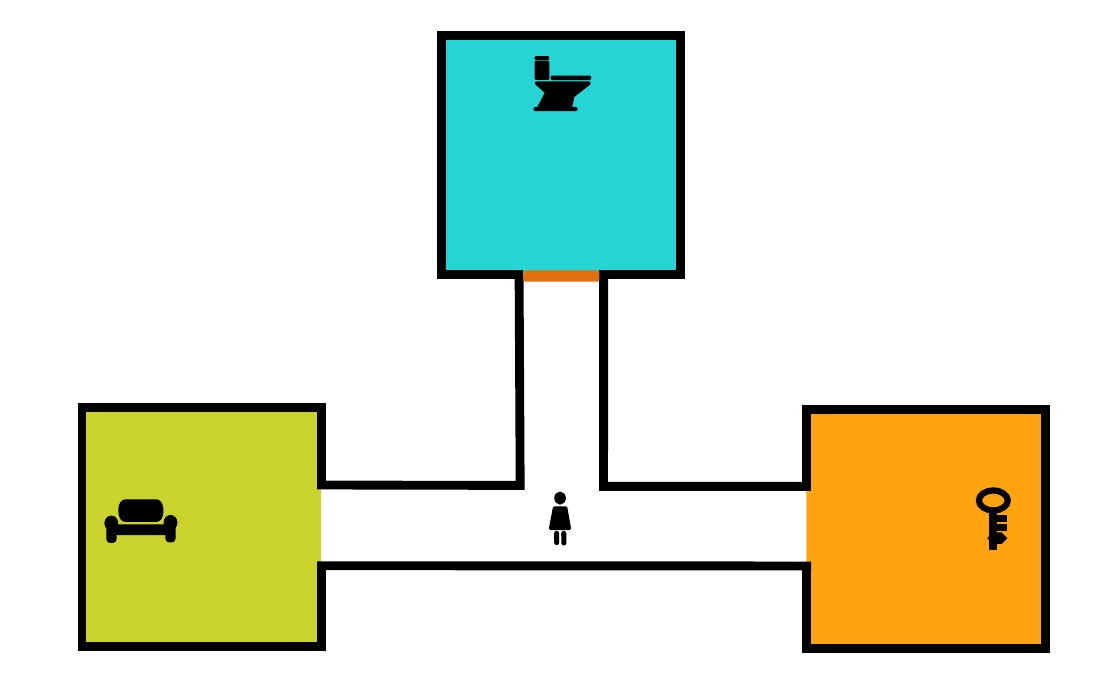} 
    \caption{A labelled Det-POMDP environment.}
    \label{fig:fig1}
\end{figure}

\section{Preliminaries and Problem Formulation}
\subsection{Deterministic Partially Observable Markov Decision Process}
In RL, an agent learns a policy by interacting with its environment. Often, the agent is assumed to fully observe the environment state, which satisfies the Markov property, meaning the future is independent of the past given the present state. However, in many realistic settings, the agent has only partial observability of the state, receiving observations that provide incomplete information.

Figure~\ref{fig:fig1} shows a partially observable environment with four areas. The agent's perception is limited to its current area, and its action space consists of four movement actions and a sit action. From its raw observation (the current area), the agent can recognize symbolic labels representing high-level, task-relevant events, such as ``key acquired``. Labels are determined by the objects within each area. The task requires a specific sequence of actions: the agent must first visit the $orangeroom$ to obtain the $key$, which is needed to enter the locked $cyanroom$. After reaching the $toilet$, it must then travel to $limegreenroom$ and $sit$ on the $sofa$ to receive a reward. The core challenge is that the agent is memoryless: based only on its current area, it cannot know if it has acquired the $key$ or visited the $toilet$. If it $sit$s on the $sofa$ prematurely, the episode ends with no reward.

This kind of environment is formalized as a Deterministic Partially Observable Markov Decision Process (Det-POMDP)~\cite{yu2024model}. A Det-POMDP is defined as a tuple $\mathcal{P} = \langle S, A, s_0, P, R, \gamma, O, Z \rangle$, where $S$, $A$, and $O$ are the finite sets of states, actions, and observations; $s_0 \in S$ is the initial state; $P: S \times A \rightarrow S$ is the deterministic transition function that maps each state-action pair to a unique next state; $R: S \times A \rightarrow \mathbb{R}$ is the reward function that assigns a real-valued reward to each state-action pair; $\gamma \in [0,1]$ is the discount factor and $Z: S \rightarrow O$ is the observation function that maps each state to the observation the agent receives in that state. Finally, AP is a set of atomic propositions for events, and $L: O \rightarrow 2^{\text{AP}}$ is a labeling function mapping observations to these events.

In a Det-POMDP, the environment begins in the initial state $s_0$. At each timestep, given the current state $s$, the agent observes $o = Z(s)$ and $l = L(o)$, then selects an action $a \in A$. Upon taking action $a$, the agent receives reward $r = R(s, a)$, and the environment transitions to $s' = P(s, a)$. Since multiple states can map to the same observation through $Z$, the agent cannot distinguish between different underlying states that produce identical observations. This partial observability means that the observation the agent perceives may not satisfy the Markov property—the same observation-action pair can lead to different future observation and reward depending on the hidden state.

The environment in Figure~\ref{fig:fig1} can be modeled by the Det-POMDP in which states $S = O \times \{0,1\} \times \{0,1\}$, where $O$ is the four areas and the binary values indicate whether the $key$ has been acquired and whether the $toilet$ has been visited. The AP is $\{toilet$, $sofa$, $key$, $None\}$. The non-Markovian property of observations becomes evident when the agent is at the $corridor$: taking action $up$ leads to $cyanroom$ if the $key$ has been acquired, but remains at $corridor$ otherwise. Since both states $(corridor, 0, *)$ and $(corridor, 1, *)$ produce the same observation $corridor$, the observation-action pair $(corridor, up)$ does not uniquely determine the next observation. The agent receives a reward of 1 when it operates $sit$ in the state $(limegreenroom, *, 1)$, and 0 otherwise.

In partially observable environments, the agent's reward often depends not just on the current observation and action, but on the entire observation-action history. For instance, in our example, the agent receives reward only when it sits on $sofa$ after reaching $toilet$ - the reward function inherently depends on the temporal ordering of past observations.

\subsection{Problem Formulation}
In Det-POMDPs, the agent must select actions based solely on observable information. A policy $\pi: (O \times A)^* \times O \rightarrow A$ maps observation-action histories to actions, where $\pi(o_0, a_0, \ldots, o_t)$ specifies the action to take after observing the sequence $o_0, a_0, \ldots, o_t$.

Given a set of traces $\mathcal{D} = \{\tau_1, \tau_2, \ldots, \tau_n\}$ collected from a Det-POMDP, where each trace $\tau_i$ consists of observation-action-reward sequences, our goal is to learn an optimal policy $\pi^*$ that maximizes the expected discounted reward $\mathbb{E}\left[\sum_t \gamma^t r_t\right]$. The challenge is that standard RL algorithms cannot be directly applied due to the non-Markovian property of rewards and transitions with respect to observations.

\subsection{Reward Machine}
Reward Machines (RMs) provide a principled approach to address this challenge by maintaining sufficient memory to transform non-Markovian reward functions into Markovian.

\begin{definition}[Reward Machine]
\label{def:rm}
Given a set of atomic propositions $\text{AP}$, observations $O$, and actions $A$, a \acrfull{RM} is a tuple $\mathcal{RM} = \langle U, u_0, \delta_U, \delta_R \rangle$, where:
\begin{itemize}
    \item $U$ is a finite set of RM states,
    \item $u_0 \in U$ is the initial RM state,
    \item $\delta_U: U \times 2^{\text{AP}} \rightarrow U$ is the RM state-transition function,
    \item $\delta_R: U \times O \times A \rightarrow \mathbb{R}$ is the reward-output function.
\end{itemize}
\end{definition}

A RM is a finite-state automaton that processes sequences of labels to determine appropriate reward function. 
The key insight of an RM is that its reward-output function, $\delta_R$, is parameterized by the RM state $u$. This allows the same observation-action pair to yield different rewards depending on the history encoded in $u$, converting history-dependent rewards into a Markovian framework.


In practice, policy learning with RMs is a two-stage process: first, an RM is inferred or otherwise constructed, and second, it's used to create an augmented, Markovian state $(o, u)$ for a standard RL agent. The inferred RM must correctly predict rewards for all feasible histories to ensure the augmented state space truly restores the Markov property and enables learning of optimal policies.

\section{Method}
\subsection{Transition Machine}
To enable policy learning in Det-POMDPs, we need to infer a RM that provides sufficient memory to restore the Markov property. A natural goal to inferring RMs from traces is to seek the minimal RM that correctly explains the observed reward patterns. For our example environment, such a minimal RM would contain only two states, as shown in Figure~\ref{fig:fig2}: one representing ``not reached $toilet$`` and the other representing ``reached $toilet$.`` This RM correctly predicts that rewards are obtained when the agent sits down on $sofa$ after reaching $toilet$, and zero rewards otherwise.

\begin{figure}[htbp]
    \centering 

    \begin{subfigure}[b]{0.4\columnwidth}
        \centering
        \includegraphics[width=\textwidth]{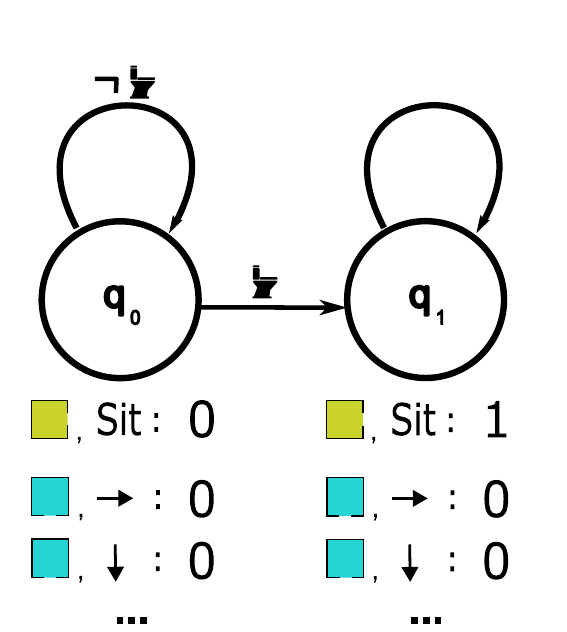}
        \caption{minimal RM}
        \label{fig:fig2}
    \end{subfigure}
    \hfill 
    \begin{subfigure}[b]{0.56\columnwidth}
        \centering
        \includegraphics[width=\textwidth]{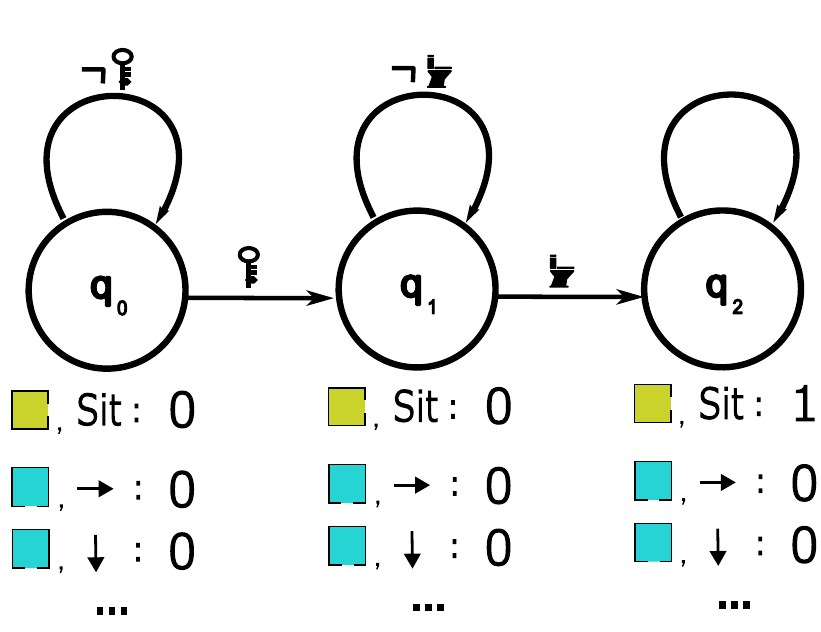}
        \caption{``optimal'' RM}
        \label{fig:fig3}
    \end{subfigure}

    \caption{Two possible RMs for the environment.}
\end{figure}

However, this minimal RM, while correctly capturing the reward dependencies, fails to provide sufficient memory for policy learning. The RM does not distinguish between states where the agent has or has not acquired $key$, yet this distinction is crucial for determining valid actions. For instance, the agent cannot enter $cyanroom$ without first obtaining $key$, but the minimal RM provides no memory of this constraint. A better RM is shown in Figure~\ref{fig:fig3}.

This limitation arises because the minimal RM focuses exclusively on reward prediction accuracy. In our example, the agent cannot enter $cyanroom$ without $key$, yet this transition constraint is completely absent from the minimal RM's representation because it is not about reward.

This problem reflects a fundamental characteristic of POMDPs from the agent's perspective. The non-Markovian behavior an agent experiences arises from two distinct sources: a reward dependency, where rewards are based on the sequence of observations and actions, and a transition dependency, where next observations are based on the same history. RMs address only the first type of dependency, leaving the second type unresolved.

As RMs maintain memory to capture non-Markovian reward dependencies into Markovian, we use an analogous mechanism for non-Markovian transition dependencies.

\begin{definition}[Transition Machine]
A Transition Machine (TM) is a tuple $\mathcal{TM} = \langle Q, q_0, \delta_Q, \delta_P \rangle$, where:
\begin{itemize}
    \item $Q$ is a finite set of TM states,
    \item $q_0 \in Q$ is the initial TM state,
    \item $\delta_Q: Q \times 2^{\text{AP}} \rightarrow Q$ is the TM state-transition function,
    \item $\delta_P: Q \times O \times A \rightarrow O$ is the transition-output function.
\end{itemize}
\end{definition}

\begin{figure}[htbp]
    \centering
    \includegraphics[width=0.4\columnwidth]{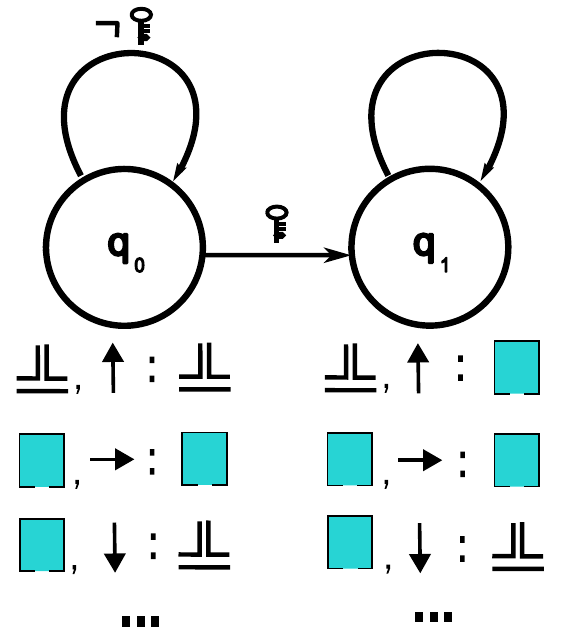} 
    \caption{The least TM for the environment.}
    \label{fig:your_label}
\end{figure}

A TM operates as a finite-state automaton that processes observation-action sequences to predict observation transitions. Starting from the initial state $q_0$, the TM transitions between states based on labels, similar to RMs. The key difference lies in the output function $\delta_P$, which predicts the next observation, rather than producing rewards. The output function $\delta_P: Q \times O \times A \rightarrow O$ is state-dependent, allowing different predictions for the same observation-action pair based on the current state $q$. This structure enables the TM to maintain memory of transition-relevant history, converting non-Markovian transition dependencies into a Markovian one where the current TM state encapsulates the necessary past information for accurate transition prediction.

\subsection{Method Framework}

Our framework addresses policy learning in Det-POMDPs through a two-stage approach: first inferring a TM and a RM that provide sufficient memory, then applying standard RL algorithms on the resulting Markovian environments.


We assume a dataset $D = \{\zeta_1, \zeta_2, \ldots, \zeta_n\}$ of traces collected from an unknown Det-POMDP. Each trace is a sequence of observation-action-reward steps, which produces a labeled trace $\zeta = \langle (l_0, o_0, a_0, r_0), \ldots, (l_n, o_n, a_n, r_n) \rangle$, where $l_t = L(o_t)$ represents symbolic labels extracted from observations $o_t$ using the given labeling function $L$.

Our approach infers both TMs and RMs from this dataset $D$. They work together to transform the Det-POMDP into an MDP. Once inferred, the augmented state space $(o, u, q)$ becomes Markovian, where $o$ is the current observation, $u$ is the RM state, and $q$ is the TM state.

To ensure the inferred automata enable policy learning, we require them to be resolvent—that is, they must accurately predict the environment's behavior.

\begin{definition}[Resolvent Machine]
Let $M$ be a Det-POMDP $\langle S, s_0, A, T, R, Z, O, \gamma \rangle$ with labeling $L$. For any automaton with transition function $\delta$, we denote by $\delta^*$ its extension to sequences: $\delta^*(q, \epsilon) = q$ and $\delta^*(q, w\cdot\sigma) = \delta(\delta^*(q, w), \sigma)$.

(1) An RM $\mathcal{RM} = \langle U, u_0, \delta_U, \delta_R \rangle$ is \textbf{resolvent} to $M$ if for every feasible history $h = (o_0, a_0, o_1, a_1, \ldots, o_t)$ in $M$, letting $u = \delta_U^*(u_0, \bar{l})$ where $\bar{l} = (L(o_1), \ldots, L(o_t))$: $$\forall a \in A: \delta_R(u, o_t, a) = R(s_t, a)$$

(2) A TM $\mathcal{TM} = \langle Q, q_0, \delta_Q, \delta_T \rangle$ is \textbf{resolvent} to $M$ if for every feasible history $h$, letting $q = \delta_Q^*(q_0, \bar{l})$: $$\forall a \in A: \delta_T(q, o_t, a) = Z(T(s_t, a))$$

where $s_t$ is the underlying state corresponding to $h$.
\end{definition}

Among all resolvent machines, we seek the minimal ones to ensure efficient computation during policy learning.

\begin{definition}[Minimal Resolvent Machine]
A TM or RM is \textbf{minimal resolvent} for a Det-POMDP if it is resolvent and there exists no resolvent TM or RM that has fewer states.
\end{definition}

Our dual automata framework naturally decomposes the inference task into finding minimal resolvent TMs and RMs. This avoids the complex formulation required when forcing all dependencies into a single RM structure.

Specifically, our method proceeds by inferring the minimal resolvent TM and minimal resolvent RM from $D$. These automata, once inferred, provide the augmented state representation necessary for applying standard RL algorithms to learn optimal policies in the original Det-POMDP.

\subsection{Dual Behavior Mealy Machine}
To address the first step of the framework, a straightforward approach would be to develop separate inference algorithms for TMs and RMs. However, since these two types of automata have similar structures, they can be modeled within a unified framework and inferred using a single algorithm.

To achieve this unified representation, we introduce the Dual Behavior Mealy Machine (DBMM), an automaton framework that subsumes both TMs and RMs.

\begin{definition}[Dual-Behavior Mealy Machine]
\label{def:dbmm}
A Dual-Behavior Mealy Machine (DBMM) is defined as a tuple $\mathcal{N} = \langle V, v_0, \mathcal{I_\alpha}, \mathcal{I_\beta}, \mathcal{O}, \mathcal{T}, \mathcal{G} \rangle$, where:
\begin{itemize}
    \item $V$ is a finite set of states, $v_0 \in V$ is the initial state,
    \item $\mathcal{I_\alpha}$ is the \textit{alpha input alphabet}. Inputs $i_\alpha \in \mathcal{I_\alpha}$ trigger outputs but do not cause state transitions.
    \item $\mathcal{I_\beta}$ is the \textit{beta input alphabet}. Inputs $i_\beta \in \mathcal{I_\beta}$ cause state transitions but do not trigger outputs. ($\mathcal{I_\alpha} \cap \mathcal{I_\beta} = \emptyset$)
    \item $\mathcal{O}$ is the output alphabet,
    \item $\mathcal{T}: V \times \mathcal{I_\beta} \rightarrow V$ is the state transition function.
    \item $\mathcal{G}: V \times \mathcal{I_\alpha} \rightarrow \mathcal{O}$ is the output function.
\end{itemize}
\end{definition}

A DBMM can represent a RM by setting $\mathcal{I_\beta} = 2^{AP}$, $\mathcal{I_\alpha} = O \times A$, and $\mathcal{O} = \mathbb{R}$. Similarly, it can represent a TM with the same input alphabets but $\mathcal{O} = O$. Consequently, this unified framework enables inferring both types of machines using a single algorithmic approach.

Beyond its unified representation, the DBMM framework offers significant computational advantages. Current methods for inferring RMs often rely on computationally expensive problem reductions. In contrast, the DBMM's structure is designed to be compatible with passive automaton learning algorithms, which build the automaton directly from trace data. This approach avoids the costly overhead of the reduction-based methods used in prior work.

\subsubsection{Problem Reduction to DBMM inference.}
With DBMM defined as a unified representation framework, we now address the automata inference step of our approach for policy learning. We reduce the problem to inferring two DBMMs: one representing the TM and another representing the RM.

This reduction involves transforming trace data into sample sets suitable for passive automaton learning. In passive automaton learning, we are given a sample set consisting of input-output sequence pairs, where each pair represents the behavior of an unknown automaton: given an input sequence, the automaton produces the corresponding output sequence. The inference task is to find a minimal automaton that can reproduce all the observed input-output behaviors.

For our problem, each labeled trace $\zeta = \langle (l_0, o_0, a_0, r_0), \ldots, (l_n, o_n, a_n, r_n) \rangle$ is converted into two distinct input-output sequence pairs:

\begin{description}
    \item[TM Sample:] Input sequence $\langle (o_0, a_0), l_0, (o_1, a_1), l_1, \ldots \rangle$ with output sequence $\langle o_1, \beta_{\mathrm{default}}, o_2, \beta_{\mathrm{default}}, \ldots \rangle$
    \item[RM Sample:] Input sequence $\langle (o_0, a_0), l_0, (o_1, a_1), l_1, \ldots \rangle$ with output sequence $\langle r_0, \beta_{\mathrm{default}}, r_1, \beta_{\mathrm{default}}, \ldots \rangle$
\end{description}

Here, observation-action pairs and labels alternate in the input sequence, while $\beta_{\mathrm{default}}$ serves as a placeholder output for label inputs, as they only trigger state transitions.

By inferring minimal DBMMs from the TM and RM sample sets respectively, we obtain DBMM representations of both the TM and RM, thus solving our inference problem and providing the memory structure needed for the subsequent policy learning stage.

\section{Algorithm for Inferring DBMMs}
\label{sec:method}

We now present our algorithm for learning DBMMs from trace data, which accomplishes the first stage of our policy learning approach by inferring the automata needed to transform the Det-POMDP into an MDP.

\renewcommand{\algorithmicrequire}{\textbf{Input:}}
\renewcommand{\algorithmicensure}{\textbf{Output:}}

\begin{algorithm}[H]
\caption{Overall Inferring Algorithm}
\begin{algorithmic}[1]
\REQUIRE TMSampleSet, RMSampleSet
\ENSURE TM, RM

\STATE TMSampleSet $\leftarrow$ ReductionRule(TMSampleSet)
\STATE TM $\leftarrow$ DB-RPNI(TMSampleSet)
\STATE RMSampleSet $\leftarrow$ StateSupp(TM, RMSampleSet)
\STATE RMSampleSet $\leftarrow$ ReductionRule(RMSampleSet)
\STATE RM $\leftarrow$ DB-RPNI(RMSampleSet)
\end{algorithmic}
\end{algorithm}

Our algorithm comprises three key stages: sample sets preprocessing, the DB-RPNI algorithm that learns minimal automata, and an observation supplement technique that leverages the inferred TM to improve RM inference.

\subsection{Preprocessing}

Sample sets often contain patterns irrelevant to state identification, complicating the inference process. We introduce two preprocessing methods to simplify the sample sets while preserving information for recovering the correct automata. These steps are applied before the inference algorithm.

\subsubsection{Redundant $\alpha$-Input Removal.}

An $\alpha$-input $i_\alpha \in \mathcal{I_\alpha}$ is redundant if it always produces the same output regardless of the automaton's state, providing no information for distinguishing states. We remove these pairs before inference and re-integrate their constant output into the final automaton. The detailed procedure is provided in Appendix~\ref{app:preprocessing}.

\subsubsection{Trivial $\beta$-Input Removal.}

In many RL environments, certain $\beta$-inputs (e.g., ``no event``) are known a priori to cause no state transitions. We call such inputs trivial. We remove these inputs from the sample data and add them back as self-loops on all states of the inferred automaton. The detailed procedure can also be found in Appendix~\ref{app:preprocessing}.

\subsection{Dual Behavior RPNI (DB-RPNI)}

We adapt the RPNI algorithm~\cite{oncina1992inferring} to infer DBMMs. Our DB-RPNI algorithm follows RPNI's two-phase structure, but incorporates specialized handling for the DBMM's dual-input structure.


\subsubsection{Phase 1: Prefix Tree Construction.}

The prefix tree transducer (PTT) is an automaton that represents all samples exactly without generalization. Each unique $\beta$-input sequence corresponds to a distinct state. The detailed pseudocode for building the PTT is provided in Appendix~\ref{app:dbrpni_pseudo}.

\subsubsection{Phase 2: State Merging.}

State merging uses the red-blue framework~\cite{von2025extending}. Red states form the core of the automaton, while blue states are merge candidates. The algorithm iteratively attempts to merge each blue state with a compatible red state. If a merge succeeds, the automaton is updated; if not, the blue state is promoted to red. This process continues until no blue states remain. The complete pseudocode is available in Appendix~\ref{app:dbrpni_pseudo}.

\subsubsection{Mergeability Criteria and Merge Operation.}

Two states can be merged if they are locally compatible and all pairs of implied merges are also locally compatible.

\begin{definition}[Local Compatibility]
Two states $u, v \in \mathcal{V}$ are locally compatible, denoted $\mathrm{Compatible}(u, v)$, iff:
\[
\mathrm{Compatible}(u, v) \iff \forall i_\alpha \in I_u \cap I_v : \mathcal{G}(u, i_\alpha) = \mathcal{G}(v, i_\alpha)
\]
where $I_u = \{i_\alpha \in I_\alpha \mid \mathcal{G}(u, i_\alpha) \text{ is defined}\}$ represents the set of $\alpha$-inputs for which state $u$ has recorded outputs.
\end{definition}

Implied merging is the transitive closure of merges that propagates throughout the automaton~\cite{von2025extending}. 
Implied merging can be handled using the folding operation~\cite{de2010grammatical}, which computes all state pairs that must be merged transitively and constructs the resulting automaton. The folding operation enables both the merge operation and mergeability checking, which it verifies by testing the local compatibility of all implied pairs.

\subsubsection{Theoretical Guarantees.}
Our algorithm inherits the theoretical guarantees of RPNI. When the input sample set satisfies the \emph{structure completeness} condition (a formal definition is provided in Appendix~\ref{app:def_structure_completeness}), our algorithm is guaranteed to infer the minimal resolvent automaton.

\begin{theorem}[Correctness, Minimality and Complexity]
\label{thm:correctness}
Let $T$ denote the number of states in the PTT, $|U|$ the number of states in the target automaton, $|L| = |\mathcal{I_\beta}|$, and $F = \max_{q \in \mathcal{V}} |\{i_\alpha : \mathcal{G}(q, i_\alpha) \text{ defined}\}|$. Then:
\begin{enumerate}
    \item If the sample set $\mathcal{S}$ is structure complete, DB-RPNI returns the minimal resolvent DBMM.
    \item With structure complete $\mathcal{S}$, DB-RPNI runs in $O(|U| \cdot |L| \cdot T \cdot F)$ time.
    \item In the general case, DB-RPNI runs in $O(T^3 \cdot F)$ time.
\end{enumerate}
\end{theorem}

\begin{proof}
The full proof is provided in Appendix~\ref{app:proof_details}.
\end{proof}

Notably, Structure Completeness is a sufficient but not necessary condition for success. In practice, the algorithm only requires the sample set to contain enough evidence to prevent incorrect state merges—a much weaker condition than full Conflict Convergence that is more easily met. 

\subsection{Observation Supplement}
Non-Markovian transition dynamics can interfere with RM inference, causing RMs to have redundant states.

\begin{example}[Transition-Reward Interference]
Consider a setting as shown in Figure~\ref{fig:StateSupp}. The agent attempting to move from $corridor$ to the $cyanroom$ will succeed only if it has previously visited room with $key$; otherwise, it remains in $corridor$. In contrast to our previous setting, the agent gets reward once it achieves $cyanroom$. In this case, the same observation-action pair can yield different rewards depending on history, which might seem like non-Markovian reward behavior. However, the root cause is non-Markovian transition behavior. Such interference introduces unnecessary complexity into reward machines.
\end{example}

\begin{figure}[htbp]
    \centering
    \includegraphics[width=0.8\columnwidth]{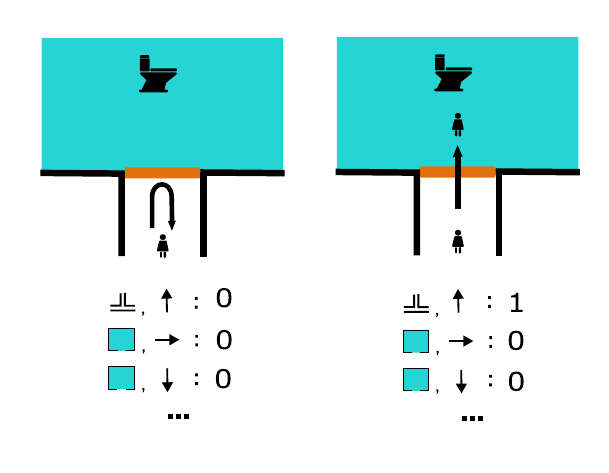} 
    \caption{Non-Markovian transitions interfering with reward machine inference.}
    \label{fig:StateSupp}
\end{figure}

To address this interference, we augment observations with TM states before RM inference. For each step in a trace, the current observation $o_i$ is paired with the TM state $q_i$ that was reached before observing $o_i$. This creates an augmented observation $o'_i = (o_i, q_i)$ that disambiguates the history. The pseudocode for this process is in Appendix~\ref{app:statesupp_pseudo}.

\begin{proposition}[Observation Supplement Effect]
If $\mathcal{TM}$ is resolvent to the Det-POMDP, then observation supplement eliminates transition-related non-Markovianity.
\end{proposition}

By decoupling transition-based non-Markovianity, we can infer more compact and interpretable RMs. These smaller automata reduce the inference complexity and improve the efficiency of subsequent policy learning.

\section{Experiments}

In this section\footnote{The source code for this paper is publicly available on GitHub at \url{https://github.com/sousoura/Inferring-Reward-Machines-and-Transition-Machines-from-POMDP.git}.}, we conduct two experiments to answer the following key research questions:
\begin{itemize}
    \item[\textbf{Q1:}] Is our method significantly more efficient at inferring automata from traces compared to SOTA baseline methods?
    \item[\textbf{Q2:}] How does each component of our pipeline contribute to inference efficiency and the quality of the automata?
    \item[\textbf{Q3:}] Can our method still infer relatively small automata in large environments?
\end{itemize}

\subsection{Efficiency Comparison with Baselines (Q1)}

\begin{figure}[htbp]
    \centering
    \includegraphics[width=0.52\textwidth]{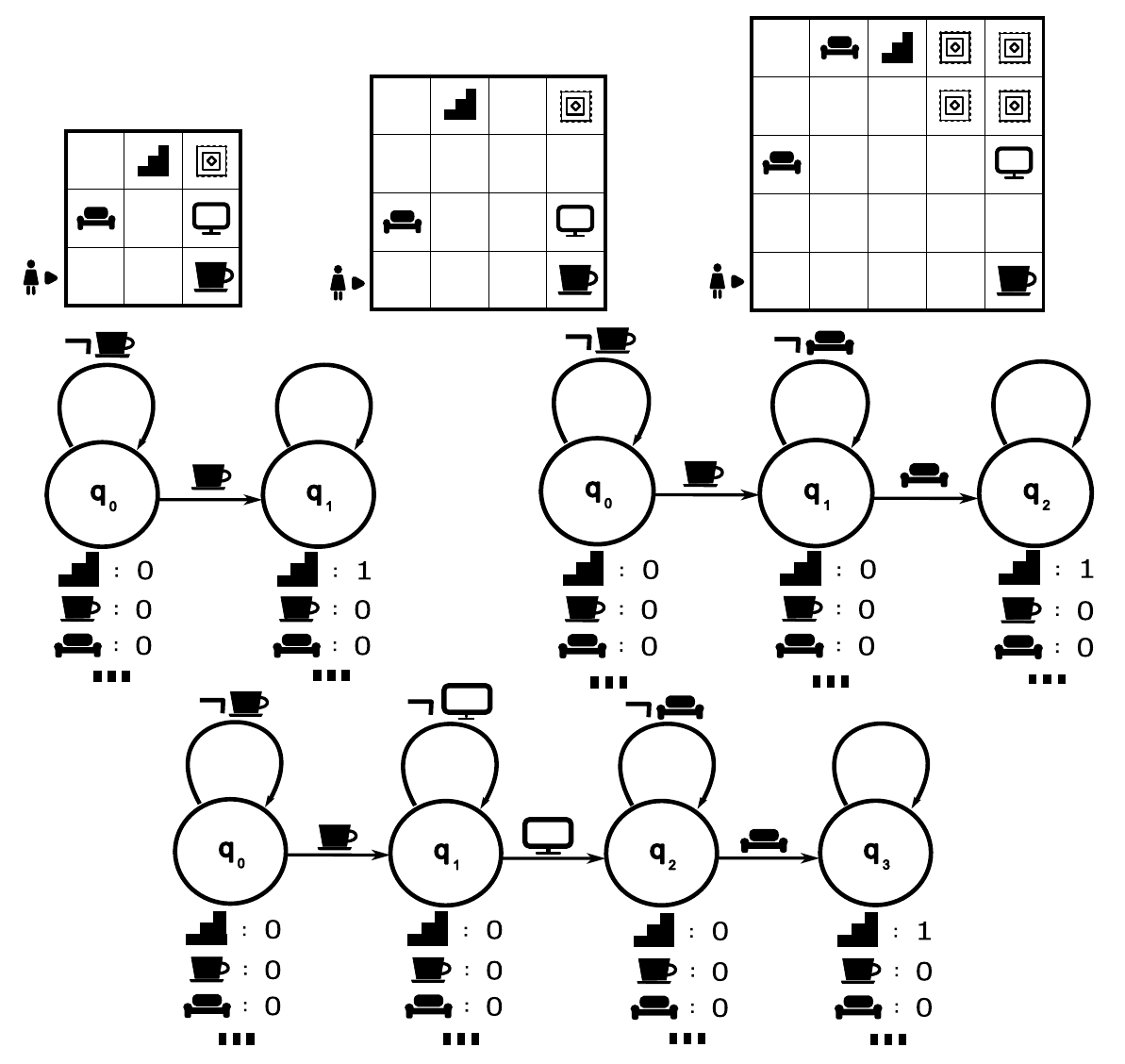} 
    \caption{The environments used in the baseline comparison.}
    \label{fig:grid}
\end{figure}

\paragraph{Baselines.}
We compare our approach against two state-of-the-art methods for inferring automata from traces:
\begin{itemize}
    \item \textbf{HMM-based approach}~\cite{abate2023learning}: A method based on Hidden Markov Models to infer DFAs.
    \item \textbf{ILP-based approach}~\cite{hyde2024detecting}: An Integer Linear Programming method that infers RMs.
\end{itemize}
We did not compare against the method in~\cite{toro2019learning}, as its method is local search-based and its runtime and quality are only dependent on hyperparameter tuning.

\paragraph{Experimental Setup.}
To answer Q1, we evaluate inference efficiency on three discrete grid environments of increasing complexity: 3×3, 4×4, and 5×5 grids with 3, 4, and 5 task phases respectively (see Figure~\ref{fig:grid}). As shown in the automata structures, agents must reach the final automaton state and perform a specific action to complete the task and receive rewards. This task structure is chosen because the HMM baseline can only handle reward patterns of this form. Further details regarding the datasets, hardware, and software can be found in Appendix~\ref{app:implementation_details}.


 \begin{table*}[htbp]
    \centering
    \small
    \caption{Ablation study on a 25$\times$25 grid environment.}
    \label{tab:ablation_study}
    \begin{tabular}{lcccccc}
        \toprule
        \multicolumn{1}{c}{\textbf{Configuration}} & \multicolumn{3}{c}{\textbf{Low Data (1,000 traces)}} & \multicolumn{3}{c}{\textbf{High Data (10,000 traces)}} \\
        \cmidrule(lr){2-4} \cmidrule(lr){5-7}
         & \textbf{TM States} & \textbf{RM States} & \textbf{Runtime (s)} & \textbf{TM States} & \textbf{RM States} & \textbf{Runtime (s)} \\
        \midrule
        Full Pipeline                        & 7 & 2 & 412.2     & 7 & 2 & 4,174.6   \\
        w/o Observation Supplement                 & 7          & 218        & 26,976.9           & 7          & 15         & 4,252.5            \\
        w/o Redundant $\alpha$-Input Removal & 7          & 2          & 3,176.2            & 7          & 2          & 30,395.8           \\
        w/o Trivial $\beta$-Input Removal    & $>$250       & --         & $>$172,800           & 8          & 2          & 65,728.3           \\
        w/o Both Removal Rules               & --         & --         & --                 & --         & --         & $>$172,800           \\
        \bottomrule
    \end{tabular}
\end{table*}

\paragraph{Results and Analysis.}
Figure~\ref{fig:Baseline} presents the comparative performance. Our approach demonstrates substantial efficiency advantages. In the 3×3 environment, our method requires only 1.3 seconds compared to 104.0 seconds for the HMM approach and 5.3 seconds for the ILP method. This efficiency gap widens in larger environments: in the 4×4 setting, our method takes 3.9 seconds, while the HMM approach requires over 6,000 seconds and the ILP method needs approximately 5,500 seconds. For the 5×5 environment, both baseline methods time out, while our approach infers the correct automaton in 56.1 seconds.

\begin{figure}[htbp]
    \centering
    \includegraphics[width=0.4\textwidth]{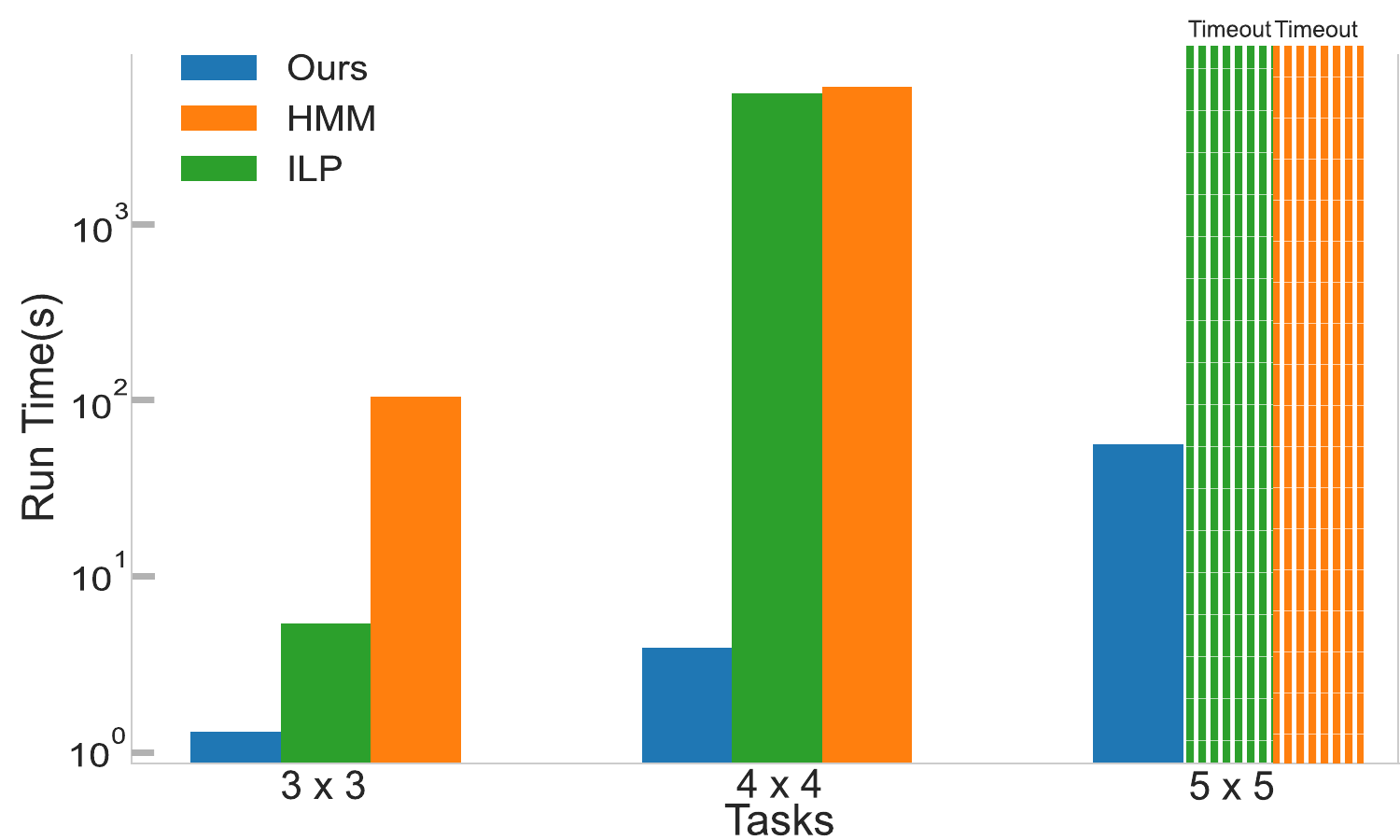} 
    \caption{Baseline Comparison.}
    \label{fig:Baseline}
\end{figure}

\subsection{Ablation and Scalability Analysis (Q2 \& Q3)}

\paragraph{Configurations.}
We test configurations to isolate each component's effect. We compare our \textbf{Full Pipeline} with ablated versions that run without \textbf{Observation Supplement}, without \textbf{Redundant $\alpha$-Input Removal}, without \textbf{Trivial $\beta$-Input Removal}, and without \textbf{Both Removal Rules}.

\paragraph{Experimental Setup.}
We conduct this study on a large-scale, randomly generated 25$\times$25 grid environment designed to test scalability. In this environment, an underlying TM controls the passability of certain locations, while a RM dictates the rewards, creating complex, non-Markovian dynamics. To analyze the impact of data availability, we test our method under two conditions: a ``Low Data`` and a ``High Data`` setting. All specific generation parameters and dataset statistics are detailed in Appendix~\ref{app:ablation_env_details}.


\paragraph{Results and Analysis.}
 The results reveal that our optimizations are crucial, with their importance becoming even more pronounced when data is scarce. Removing Observation supplement increases the inferred RM complexity fivefold with high data (from 3 to 15 states), and this effect is drastically exacerbated in the low-data setting, where the state count explodes to 218. Similarly, the preprocessing steps are vital for efficiency. Removing Redundant $\alpha$-Input Removal or Trivial $\beta$-Input Removal results in significant 7x and 15x slowdowns, respectively, with high data. The latter's removal even causes the process to fail entirely with low data. In contrast, our full pipeline demonstrates strong scalability and robustness.
 This success in a complex 25x25 environment, where baselines from Experiment 1 failed on smaller grids, confirms the method's practical applicability.

\subsection{Validation on Downstream Policy Learning}
To ultimately validate the practical utility of our approach, we integrated the automata inferred by our full pipeline (7 TM states and 2 RM states) with a standard Q-learning agent to solve the control task. The agent's state was defined as the augmented tuple $(o, u, q)$, where $o$ is the raw observation from the environment, and $u$ and $q$ are the current states of the inferred RM and TM, respectively. When trained in the 25$\times$25 partially observable environment, the agent converged to the optimal policy. This final result demonstrates that our method produces automata that effectively restore the Markov property, enabling standard RL algorithms to solve complex, non-Markovian decision-making problems.

\section{Conclusion}

This paper introduces Transition Machines (TMs) to disentangle transition-based non-Markovianity from reward dependencies in POMDPs. We then propose the Dual Behavior Mealy Machine (DBMM), a unified framework for representing both TMs and Reward Machines, along with an efficient inference algorithm to learn them from traces. Our approach restores the Markov property to the environment, enabling standard RL algorithms to solve complex tasks.

Despite these contributions, our work has limitations that suggest clear avenues for future research. First, our method is currently restricted to deterministic POMDPs. We view this focus as a crucial foundational step, as successfully decoupling dependencies in the deterministic setting lays the groundwork for future extensions to stochastic environments, for instance, by integrating probabilistic automata. Second, the algorithm's performance relies on high-quality trace data. Future work on methods with less stringent data assumptions, such as inferring from incomplete data, is crucial for making these automaton-based approaches practical for real-world applications.

\bibliography{aaai2026}

\newpage

\begin{center}
    {\LARGE \textbf{Appendix}}
\end{center}
\addcontentsline{toc}{section}{Appendix}

\section{Detailed Algorithms}
\label{app:alg_details}

This appendix contains supplementary algorithmic details, including procedures for preprocessing and pseudocode for the DB-RPNI and Observation Supplement algorithms.

\subsection{Redundant $\alpha$-Input Removal}
\label{app:preprocessing}
\begin{description}[leftmargin=0pt,labelindent=0pt,font=\normalfont\itshape]
\item[\textbf{Preprocessing Operation:}]
For each redundant $i_\alpha$ with constant output $o^*$:
\begin{enumerate}[leftmargin=*,topsep=0.5ex]
    \item Create reduced sample set $\mathcal{S}'$ by removing all occurrences of $(i_\alpha, o^*)$ from each sample $(w, o) \in \mathcal{S}$.
    \item Record the mapping $\mathcal{R}(i_\alpha) = o^*$.
\end{enumerate}

\item[\textbf{Recovery:}] 
After inferring the DBMM from $\mathcal{S}'$, construct the complete DBMM by setting $\mathcal{G}(q, i_\alpha) = \mathcal{R}(i_\alpha)$ for all $q \in \mathcal{V}$ and all recorded $i_\alpha$.
\end{description}

\subsection{Trivial $\beta$-Input Removal}
\begin{description}[leftmargin=0pt,labelindent=0pt,font=\normalfont\itshape]
\item[\textbf{Preprocessing Operation:}]
Given a set of known trivial inputs $\mathcal{I'_\beta} \subseteq \mathcal{I_\beta}$:
\begin{enumerate}[leftmargin=*,topsep=0.5ex]
    \item Create reduced sample set $\mathcal{S}'$ where for each $(w, o) \in \mathcal{S}$, remove all $i_\beta \in \mathcal{I'_\beta}$ from $w$ and their corresponding outputs from $o$.
\end{enumerate}

\item[\textbf{Recovery:}] 
After inferring DBMM $\mathcal{M}'$ from $\mathcal{S}'$, add self-loop transitions $\mathcal{T}(q, i_\beta) = q$ for all $q \in \mathcal{V}$ and all $i_\beta \in \mathcal{I'_\beta}$.
\end{description}

\subsection{DB-RPNI Pseudocode}
\label{app:dbrpni_pseudo}

The DB-RPNI algorithm consists of two main phases: building a Prefix Tree Transducer (PTT) from the sample data, and then iteratively merging states to generalize the automaton.

\begin{algorithm}[H]
\caption{BuildPTT}
\begin{algorithmic}[1]
\REQUIRE{Sample set $\mathcal{S}$, $\mathcal{\mathcal{I_\alpha}}, \mathcal{\mathcal{I_\beta}}, \mathcal{O}$}
\ENSURE{PTT $M$}

\STATE $M \leftarrow \langle \mathcal{V}, q_0, \mathcal{\mathcal{I_\alpha}}, \mathcal{\mathcal{I_\beta}}, \mathcal{O}, \mathcal{T}, \mathcal{G} \rangle$; $\mathcal{V} \leftarrow \{q_0\}$     
\FOR{each sample $(x, y)$ in $\mathcal{S}$}
    \STATE $q \leftarrow q_0$
    \FOR{each $x_{i}, y_{i}$ in $x, y$}
        \IF{$x_i \in I_\alpha$}
            \STATE $\mathcal{G}(q, x_i) \leftarrow y_i$
        \ELSIF{$x_i \in I_\beta$}
            \IF{$\mathcal{T}(q, x_i)$ is undefined}
                \STATE Create $q'$ as new state
                \STATE $\mathcal{V} \leftarrow \mathcal{V} \cup \{q'\}$; $\mathcal{T}(q, x_i) \leftarrow q'$
            \ENDIF
            \STATE $q \leftarrow \mathcal{T}(q, x_i)$
        \ENDIF
    \ENDFOR
\ENDFOR

\RETURN{$M$}
\end{algorithmic}
\end{algorithm}

\begin{algorithm}[H]
\caption{StateMerging}
\begin{algorithmic}[1]
\REQUIRE PTT $\mathcal{M} = \langle \mathcal{V}, q_0, \mathcal{I_\alpha}, \mathcal{I_\beta}, \mathcal{O}, \mathcal{T}, \mathcal{G} \rangle$
\ENSURE Minimal DBMM $\mathcal{M}'$
\STATE $Red \gets \{q_0\}$
\STATE $Blue \gets \{q \in \mathcal{V} : \exists r \in Red, l \in \mathcal{I_\beta}, \mathcal{T}(r, l) = q\}$
\WHILE{$Blue \neq \emptyset$}
    \STATE Select $q_b \in Blue$
    \STATE $merged \gets$ false
    \FOR{each $q_r \in Red$}
        \IF{$\textsc{Mergeable}(q_r, q_b, \mathcal{M})$}
            \STATE $\mathcal{M} \gets \textsc{Merge}(q_r, q_b, \mathcal{M})$
            \STATE $merged \gets$ true
            \STATE \textbf{break}
        \ENDIF
    \ENDFOR
    \IF{$not\ merged$}
        \STATE $Red \gets Red \cup \{q_b\}$
    \ENDIF
    \STATE $Blue \gets \{q \in \mathcal{V} : \exists r \in Red, l \in \mathcal{I_\beta}, \mathcal{T}(r, l) = q \wedge q \notin Red\}$
\ENDWHILE
\RETURN $\mathcal{M}$
\end{algorithmic}
\end{algorithm}

As noted in the main text, the \textsc{Mergeable} check and the \textsc{Merge} operation are implemented using the folding operation from \cite{de2010grammatical}.

\subsection{Observation Supplement Algorithm}
\label{app:statesupp_pseudo}
The observation supplement algorithm augments observation-based traces with the state information from a inferred Transition Machine (TM) to resolve transition-based non-Markovianity before inferring a Reward Machine (RM).

\begin{algorithm}[H]
\caption{Observation Supplement}
\begin{algorithmic}[1]
\REQUIRE Inferred TM $\mathcal{TM} = \langle Q, q_0, \delta_Q, \delta_T \rangle$, \\
Trajectory $\zeta = \langle (\ell_0, o_0, a_0, r_0), (\ell_1, o_1, a_1, r_1), \ldots \rangle$
\ENSURE Augmented trajectory $\zeta'$

\STATE $q \gets q_0$
\STATE $\zeta' \gets \langle \rangle$
\FOR{each $(\ell_i, o_i, a_i, r_i)$ in $\zeta$}
    \STATE $o'_i \gets (o_i, q)$
    \STATE $\zeta' \gets \zeta' \cdot (\ell_i, o'_i, a_i, r_i)$
    \STATE $q \gets \delta_Q(q, \ell_i)$
\ENDFOR
\RETURN $\zeta'$
\end{algorithmic}
\end{algorithm}

\section{Details for Theoretical Guarantees}
\label{app:guarantees}

This appendix provides the formal definition of structure completeness and the detailed proof for the theorem presented in Section 5.2.4.

\subsection{Formal Definition of Structure Completeness}
\label{app:def_structure_completeness}

\begin{definition}[Structure Completeness]
When the underlying DBMM is $\langle \mathcal{V}, q_0, \mathcal{I_\alpha}, \mathcal{I_\beta}, \mathcal{O}, \mathcal{T}, \mathcal{G} \rangle$ and sample set $\mathcal{S} \subseteq ((\mathcal{I_\alpha} \cup \mathcal{I_\beta})^* \times \mathcal{O}^*)$ consists of input-output sequence pairs. For any input sequence $w$ and state $q$, let $\mathcal{T}^*(q, w)$ denote the state reached from $q$ by processing only the $\beta$-inputs in $w$. Two input sequences $u, v$ conflict given $\mathcal{S}$, denoted $\mathrm{Conflict}_\mathcal{S}(u, v)$, if the sample set contains evidence that the states reached by $u$ and $v$ produce different outputs for a same $\alpha$-input. We say $\mathcal{S}$ is \emph{structure complete} if and only if it satisfies the following three conditions:

\begin{itemize}
    \item \textbf{State Coverage:} $\forall q \in \mathcal{V}$, $\exists (w, o) \in \mathcal{S}$, $\exists i$: $\mathcal{T}^*(q_0, w[1:i]) = q$

    \item \textbf{Transition Coverage:} $\forall q_1, q_2 \in \mathcal{V}$, $\forall l \in \mathcal{I_\beta}$: $\mathcal{T}(q_1, l) = q_2 \rightarrow$ $\exists (w, o) \in \mathcal{S}$, $\exists i$: $w[i] = l \wedge \mathcal{T}^*(q_0, w[1:i-1]) = q_1$

    \item \textbf{Conflict Convergence:} $\forall (w_1, o_1), (w_2, o_2) \in \mathcal{S}$, $\forall i, j$: $\mathcal{T}^*(q_0, w_1[1:i]) \neq \mathcal{T}^*(q_0, w_2[1:j]) \rightarrow \mathrm{Conflict}_\mathcal{S}(w_1[1:i], w_2[1:j])$
\end{itemize}
\end{definition}

\subsection{Proof of Theorem 5.2.4}
\label{app:proof_details}

\begin{proof}[Proof of Theorem]
(1) follows from the RPNI correctness theorem~\cite{oncina1992inferring} applied to DBMMs. The structure completeness condition ensures that the sample set contains enough information to distinguish any two non-equivalent states and cover all necessary transitions, preventing incorrect merges and guaranteeing that the final automaton is isomorphic to the target minimal automaton.

(2) With structure completeness, the red-blue framework ensures at most $|U| \cdot |L|$ merge attempts are promoted to red states, where $|U|$ is the number of states in the target automaton. Each merge attempt involves checking compatibility, which is handled by the folding operation. The folding operation traverses the state pairs to be merged. In the best case where no incorrect merges are considered, this process runs on the order of the size of the Prefix Tree Transducer (PTT), $T$. Each compatibility check examines at most $F$ input-output pairs. Thus, the complexity is $O(|U| \cdot |L| \cdot T \cdot F)$.

(3) In the general case, without the guarantee of structure completeness, the algorithm may attempt to merge any pair of states from the PTT. The number of possible merge attempts in the red-blue framework is bounded by $O(T^2)$. Each merge requires $O(T \cdot F)$ operations for compatibility checking and folding, leading to a worst-case complexity of $O(T^3 \cdot F)$.
\end{proof}

\subsection{Proof of Proposition on Observation Supplement Effect}
\label{app:proof_observation_supplement}

\begin{proposition}[Observation Supplement Effect]
If the Transition Machine $\mathcal{TM}$ is resolvent with respect to the Det-POMDP $\mathcal{P}$, then applying the observation supplement procedure eliminates transition-related non-Markovianity with respect to the augmented observations.
\end{proposition}

\begin{proof}
Let the Det-POMDP be $\mathcal{P} = \langle S, A, s_0, P, R, \gamma, O, Z \rangle$ with a labeling function $L$. Let the resolvent TM be $\mathcal{TM} = \langle Q, q_0, \delta_Q, \delta_T \rangle$.

Transition-related non-Markovianity means that for a given observation-action pair $(o_t, a_t)$, the next observation $o_{t+1}$ is not uniquely determined, as it depends on the history of interactions. Our goal is to prove that after observation supplement, the transition to the next augmented observation is deterministic given the current augmented observation and action.

The observation supplement procedure creates an augmented observation $o'_t = (o_t, q_t)$ at each timestep $t$. Here, $o_t$ is the raw observation from the environment, and $q_t$ is the state of the TM. According to the Observation Supplement algorithm (Algorithm 5), the TM state $q_t$ is the result of the history of labels up to time $t-1$:
$$q_t = \delta_Q^*(q_0, \langle L(o_0), L(o_1), \ldots, L(o_{t-1}) \rangle)$$
where $\delta_Q^*$ is the extension of $\delta_Q$ to sequences.

We need to show that the next augmented observation, $o'_{t+1} = (o_{t+1}, q_{t+1})$, is uniquely determined by the current augmented observation-action pair, $(o'_t, a_t)$. We can prove this by showing that each component of $o'_{t+1}$ is uniquely determined.

\paragraph{1. Determinism of the next raw observation $o_{t+1}$.}
The next raw observation $o_{t+1}$ is determined by the underlying state transition in the Det-POMDP. Let $s_t$ be the (unobserved) state of the Det-POMDP at time $t$ that corresponds to the history leading to the observation $o_t$. By the definition of a Det-POMDP:
$$s_{t+1} = P(s_t, a_t)$$
$$o_{t+1} = Z(s_{t+1}) = Z(P(s_t, a_t))$$
The core assumption is that the TM is \textbf{resolvent}. By Definition 4.3, this means that for any feasible history leading to state $s_t$ and TM state $q_t$, the following holds for any action $a_t \in A$:
$$\delta_T(q_t, o_t, a_t) = Z(P(s_t, a_t))$$
By combining these two facts, we get:
$$o_{t+1} = \delta_T(q_t, o_t, a_t)$$
The current augmented observation $o'_t = (o_t, q_t)$ provides both $o_t$ and $q_t$. The action taken is $a_t$. Since the TM's transition-output function $\delta_T$ is a deterministic function, the next raw observation $o_{t+1}$ is uniquely determined by $(o_t, q_t, a_t)$, which is contained in the augmented observation-action pair $(o'_t, a_t)$.

\paragraph{2. Determinism of the next TM state $q_{t+1}$.}
The TM updates its state based on its current state and the current label. The next TM state, $q_{t+1}$, is calculated as:
$$q_{t+1} = \delta_Q(q_t, L(o_t))$$
The current augmented observation $o'_t = (o_t, q_t)$ provides both the current observation $o_t$ and the current TM state $q_t$. The labeling function $L$ is deterministic, and the TM's state-transition function $\delta_Q$ is also deterministic. Therefore, $q_{t+1}$ is uniquely determined by the components of $o'_t$.

\paragraph{Conclusion.}
Since both components of the next augmented observation $o'_{t+1} = (o_{t+1}, q_{t+1})$ are uniquely determined by the current augmented observation-action pair $(o'_t, a_t) = ((o_t, q_t), a_t)$, the transition function in the augmented observation space is deterministic. This means the augmented process satisfies the Markov property with respect to transitions, and the observation supplement procedure has successfully eliminated the transition-related non-Markovianity.
\end{proof}

\section{Experiment Details}
\label{app:implementation_details}

This section provides supplementary details of experiment in Section 7.

\subsection{Hardware and Software}
All experiments were conducted on a machine equipped with an Intel Core i7-9750H CPU @ 2.60GHz. Our proposed method is implemented in Python. For the baseline methods, the core components of the HMM-based approach are implemented in C++ and the ILP-based approach is implemented in Python.

\subsection{Dataset Details of Trace Generation in Baselines Comparison} All trace data was generated by a random agent.
\begin{itemize}
    \item For the 3$\times$3 environment, we used 275 traces with an average length of 26.7 steps per trace.
    \item For the 4$\times$4 environment, we used 500 traces with an average length of 124.4 steps per trace.
    \item For the 5$\times$5 environment, we used 6000 traces with an average length of 257.6 steps per trace.
\end{itemize}

\paragraph{Evaluation Metric.}
To compare the performance of our approach against the baselines, the primary metric used for this experiment is \textbf{inference time}, measured in seconds. This metric directly assesses the computational efficiency of each method in constructing an automaton from the provided traces.

\subsection{Ablation Study Environment Generation}
\label{app:ablation_env_details}

The 25$\times$25 grid environment used in the ablation study was randomly generated according to the following procedure. First, we defined a ground-truth Transition Machine with 7 states and a Reward Machine with 3 states, both with randomly generated topologies. We then placed 5 distinct labels at random odd-coordinate positions within the grid. These labels trigger state transitions in the TM and RM. 

The environment's rules were tied to these machines: the TM's state determines whether certain labeled positions are passable, and the RM's state dictates the reward value received at those locations. All such rules were also randomly generated.

The datasets were generated using a random agent. The ``Low Data`` setting consists of 1,000 traces, and the ``High Data`` setting consists of 10,000 traces. The average length of a trace across both datasets is 225.3 steps.

\paragraph{Evaluation Metrics.}
For the ablation study, we evaluate the performance of our method using two key metrics. \textbf{Inference time} is used to measure the algorithm's efficiency. The \textbf{number of states} in the inferred automaton is used to assess the compactness and simplicity of the learned model.

\paragraph{Transition Machine Generation.} The state of the TM determines which positions in the environment are impassable. The generation process ensures its structural soundness and complexity through the following steps:
\begin{itemize}
    \item \textbf{Strong Connectivity Guarantee:} To ensure that every state can reach any other state, we first construct an initial cycle that includes all states. We then add extra random transitions between states, including self-loops, with a predefined probability. This process guarantees that the automaton is \textbf{strongly connected}, preventing any isolated parts of the state space.
    
    \item \textbf{Determinism and Completeness:} The automaton is deterministic, meaning for any given state and label, there is at most one defined subsequent state. During generation, we ensure each state-label pair has a unique transition. If a transition for a label is not defined in a state, we assign a random target state, thus making the automaton \textbf{complete}.
    
    \item \textbf{Impassable Area Definition:} Each TM state is associated with a set of impassable grid positions. For each state, we randomly designate approximately 30\% of the labeled positions as impassable. To ensure the environment remains solvable, an additional check guarantees that \textbf{no single position is impassable across all states simultaneously}. This ensures that traversable paths always exist, regardless of the agent's internal state.
\end{itemize}

\paragraph{Reward Machine Generation.} The RM is designed to model tasks with sequential goals while preventing exploitation of infinite reward loops.
\begin{itemize}
    \item \textbf{Directed Acyclic Graph (DAG) Structure:} The core of the RM is a Directed Acyclic Graph. States are linearly ordered (e.g., $q'_0, q'_1, \dots, q'_n$), and initial transitions only permit movement from a lower-indexed state to a higher-indexed one. This establishes a clear forward progression towards a designated \textbf{terminal state}, which is the sole accepting state.
    
    \item \textbf{Ensuring Reachability to Terminal State:} After the initial DAG construction, we verify that at least one path exists from every non-terminal state to the terminal state. If a state is found to be a dead end, a direct transition to the terminal state is added.
    
    \item \textbf{Negative Reward Cycles:} To add complexity and penalize suboptimal behavior, we introduce \textbf{back edges} (transitions from higher-indexed to lower-indexed states). To prevent infinite reward exploits, these back edges are assigned a punitive negative reward. The magnitude of this negative reward is calculated to be greater than the maximum possible positive reward that can be accumulated in any path from the target of the back edge back to its source. This ensures all cycles yield a net negative reward.
    
    \item \textbf{Terminal Reward:} All transitions leading directly to the final terminal state are \textbf{guaranteed to have a positive reward}, providing a clear incentive for task completion.
\end{itemize}

\subsection{Policy Learning with Augmented State Representation}
\label{app:rl_framework_details}

To experimentally validate that the inferred Transition Machine and Reward Machine provide a sufficient state representation for policy learning, we conducted an experiment in the 25$\times$25 grid environment detailed in Appendix~\ref{app:ablation_env_details}.

\paragraph{Experimental Design.}
In this experiment, we put the agent into the 25$\times$25 grid environment, where the agent learns a policy over an augmented state space where each state $s'$ is a tuple:
$$s' = (o, q, u)$$
Here, $o \in O$ is the direct observation, $q \in Q$ is the current state of the inferred TM, and $u \in U$ is the current state of the inferred RM. At each time step, the TM and RM states are updated based on the label emitted at the agent's new location, thus equipping the agent with the necessary memory of historical context.

\paragraph{Procedure.}
We employed a standard Q-learning agent to learn the value function $Q(s', a)$. The agent was trained for 1500 episodes, interacting with the environment using an $\epsilon$-greedy exploration strategy. The key hyperparameters for the experiment were set as follows:
\begin{itemize}
    \item Learning Rate ($\alpha$): 0.1
    \item Discount Factor ($\gamma$): 0.95
    \item Initial Epsilon ($\epsilon$): 0.3, decaying by a factor of 0.995 after each episode to a minimum of 0.01.
\end{itemize}
At each step, the agent selected an action based on its current Q-table. The environment returned the next observation and a reward, and the agent's internal TM and RM models transitioned based on the observed label, producing the next augmented state $s'_{t+1}$. This experience tuple $(s'_t, a_t, r_t, s'_{t+1})$ was then used to update the Q-table.

\paragraph{Results and Conclusion.}
The agent's performance throughout the training process is shown in Figure~\ref{fig:learning_curve}. The learning curve, which plots the moving average of total reward per episode, shows a clear improvement.

\begin{figure}[htbp]
    \centering
    \includegraphics[width=1.0\columnwidth]{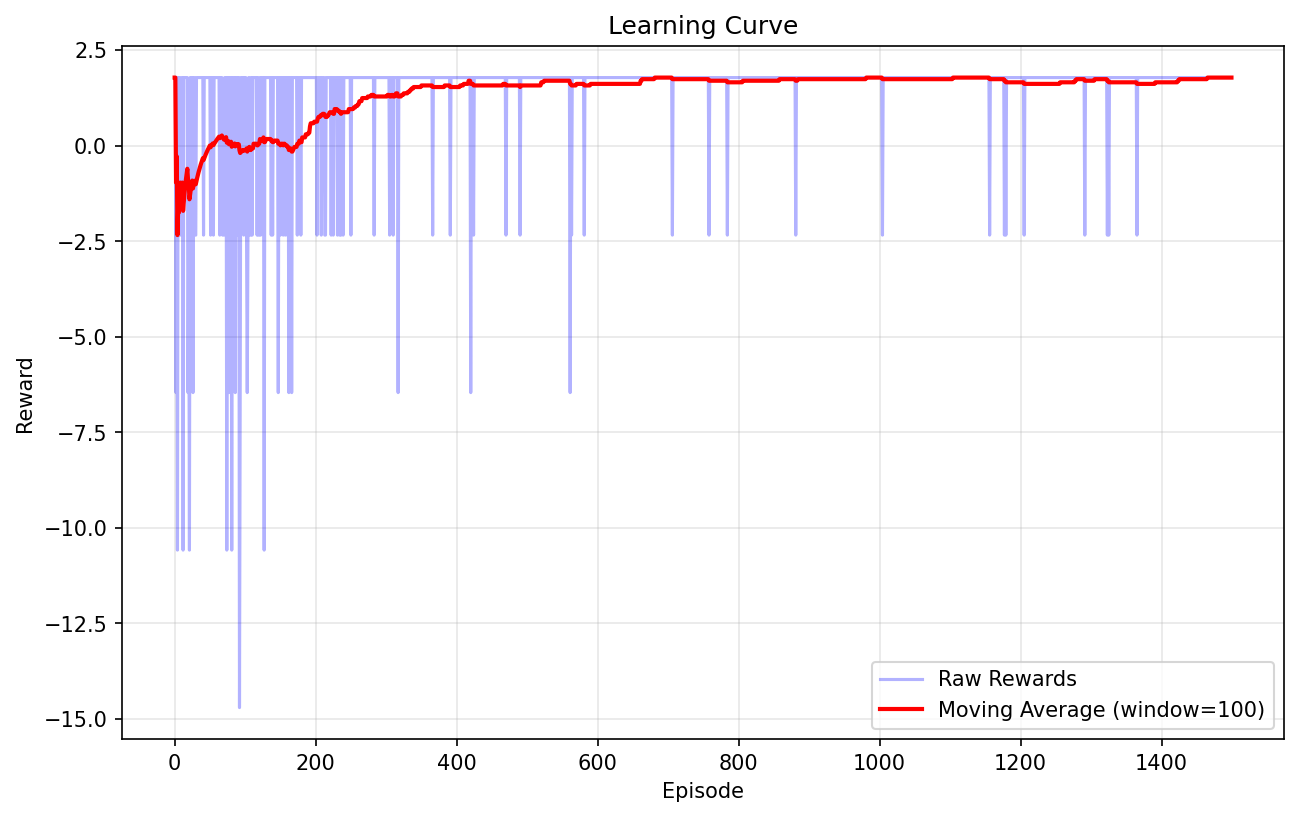} 
    \caption{The learning curve of the agent in the 25$\times$25 environment.}
    \label{fig:learning_curve}
\end{figure}

\end{document}